\def\year{2017}\relax
\documentclass[letterpaper]{article} %DO NOT CHANGE THIS
\usepackage{ijcai18}  %Required
\usepackage{times}  %Required
\usepackage{helvet}  %Required
\usepackage{courier}  %Required
\usepackage{url}  %Required
\usepackage{graphicx}  %Required
\usepackage[dvipsnames]{xcolor}
\usepackage{amsopn}
\usepackage{amsmath}
\usepackage{amsfonts}
\usepackage{amssymb}
\usepackage{algorithm}
\usepackage[noend]{algpseudocode}
\usepackage{subcaption}
\usepackage{booktabs}
\usepackage{multirow}
\usepackage{amsthm}
\usepackage{textcomp}
\usepackage{bm}
\DeclareMathOperator*{\argmin}{arg\,min}

\newcommand{\lattype}{Proposition Conserving}
\newtheorem{theorem}{Theorem}
\theoremstyle{definition}
\newtheorem{defn}{Definition} % definition numbers are dependent on theorem numbers
\newtheorem{prop}{Proposition} 
\newtheorem{exmpl}{Example} % definition numbers are dependent on theorem numbers

\makeatletter
\def\BState{\State\hskip-\ALG@thistlm}
\makeatother

\begin{document}
\newcommand{\sscomment}[1]{\textcolor{red}{[Sid] #1}}

\title{Hierarchical Expertise-Level Modeling for User Specific Robot-Behavior Explanations}
% * <sidsrivast@gmail.com> 2017-11-17T17:58:33.925Z:
% 
% > Explanations over hierarchy of model abstractions
% Needs work... We need to say something about the process of explaining  autonomous behavior by interactively concretizing user understanding/knowledge/models
% 
% explaining autonomous behavior/plans through interactive model assessment and concretization
% 
% 
% Searching for  explanations through interactive knowledge assessment and refinement
% 
% Plan explanation as search  in a hierarchical model space
% 
% 
% 
% ^.

\author{
Sarath Sreedharan \and Siddharth Srivastava \and Subbarao Kambhampati \\
School of Computing, Informatics, and Decision Systems Engineering\\
Arizona State University, Tempe, AZ 85281 USA\\
%{\tt \{ tchakra2, ssreedh3, yzhan442, rao  \} @ asu.edu}
{ \{ ssreedh3, siddharths, rao  \} @ asu.edu}
}
%\nocopyright
\maketitle
\begin{abstract}
There is a growing interest within the AI research community to develop autonomous systems capable of explaining their behavior to users. One aspect of the explanation generation problem that has yet to receive much attention is the task of explaining plans to users whose level of expertise differ from that of the explainer. We propose an approach for addressing this problem by representing the user's model as an abstraction of the domain model that the planner uses. We present algorithms for generating minimal explanations in cases where this abstract human model is not known. 
We reduce the problem of generating explanation to a search over the space of abstract models and investigate possible greedy approximations for minimal explanations. We also empirically show that our approach can efficiently compute explanations for a variety of problems.
\end{abstract}

\section{Introduction}
AI systems have the potential to transform society by assisting humans in diverse situations ranging from extraplanetary exploration to assisted living. In order to achieve this potential, however, humans working with such systems need to be able to understand them just as they would understand human team members. This presents a number of challenges because most humans do not understand AI algorithms and their behavior at the same intuitive level that they understand other humans. Recently, there have been attempts to bridge this gap by developing systems capable of explaining its behavior. Most recently \cite{explain} formulated the problem of generating explanations for plans as that of model reconciliation. Their approach relied on identifying ways of bringing the human model (i.e the explainee model) closer to the robot model so that the plan in question appears optimal in the new model. Their work looked at scenarios in which the human used a model of the domain that was at the same level of fidelity as the one used by the agent to generate the plan. This approach, unfortunately, did not capture scenarios where the human possessed a lower level of expertise and thus used a more ``abstract" or coarser representation of the model as compared to the AI agent.

In this paper, we propose a new approach to this problem where the agent explains its ongoing or planned behavior to humans with arbitrary levels of expertise. We consider explanations in the framework of counterfactual reasoning, where a user who is confused by the agent\textquotesingle s activity (or proposed activity) presents alternative behavior that they would have expected the agent to execute. This aligns with the widely held belief that humans expect explanations to be contrastive \cite{miller}. In keeping with the terminology used in social sciences literature we will call the set of alternative behavior as a {\em foil} to the proposed robot behavior.

For instance, consider a mission-control operator who needs to supervise the activity of an autonomous robot on Mars in the midst of a sandstorm that could present valuable data for analysis. If the robot proposes to go back to the base before going to a vantage point for observing the storm, the operator would naturally be perplexed, and may be motivated to ask, why doesn\textquotesingle t the robot go directly to the vantage point?! Similarly, a human team member at a manufacturing plant may be perplexed by a robot\textquotesingle s unnecessary detours while assembling an automobile engine. Not only do such situations involve personnel with varying skill levels, they also place a premium on the size of explanations. 

A natural interaction would have the robot present an explanation about why the human\textquotesingle s counterfactual suggestion would not apply in the current situation. This explanation could involve facts about the environment as well as about the robot\textquotesingle s constraints. E.g., ``I need to get a new battery pack to observe the sandstorm for at least 30 minutes without interruption". Such explanations need to be attuned to the level of understanding of the human involved. If the operator happens to be the lead designer of the robot\textquotesingle s sequential decision-making engine, the robot could provide more specific information, e.g. ``I am carrying battery-pack \#00920", because this operator knows that some battery packs wouldn\textquotesingle t allow it to carry out the full observation.

In this paper we present the \textbf{Hierarchical Expertise-Level Modeling} or the \textbf{HELM} approach for facilitating such context and user-specific explanations. HELM generates the appropriate explanation by searching through a \emph{model lattice} of possible abstractions of the agent\textquotesingle s model. Each model within this lattice represents a different level of understanding of the task, with the highest fidelity representation (corresponding to the most detailed understanding of the domain used by the robot) forming the base of the lattice and the model representing the most naive understanding of the task (for example one held by a lay user) forming the highest node. We assume that the user\textquotesingle s understanding of the domain will align with one of these abstracted models. 

%Our approach for facilitating such context and user-specific explanations involves a search through a \emph{model lattice} that contains possible abstractions of the agent\textquotesingle s model. Each model within this lattice represents a different level of understanding of the task, with the highest fidelity representation (corresponding to the most detailed understanding of the domain used by the robot) forming the base of the lattice and the model representing the most naive understanding of the task (for example one held by a lay user) forming the highest node. We assume that the user\textquotesingle s understanding of the domain will align with one of these abstracted models. 

Our explanations consist of information that may be absent in the user\textquotesingle s abstract model, and show why the foil doesn\textquotesingle t apply in the true situation. These explanations will cause the user's model to shift to a more accurate model in the lattice (and ultimately achieve model reconciliation). We will refer to model updates constituting these explanations as \emph{model concretizations}. Our framework can also be extended to situations where a user\textquotesingle s understanding is abstract and erroneous. In this paper, we focus on the fundamental aspects of the problem and restrict our attention to settings where the user\textquotesingle s  understanding is a sound abstraction of the actual situation. Since the user\textquotesingle s level of expertise is unknown to the agent, our algorithm estimates the human model before searching for an explanation.

The rest of this paper is structured as follows. In section 2, we present the formal framework we use to study this problem. Section 3 will cover different approaches for generating explanation and in section 4 we present empirical evaluation of these methods on standard IPC domains. Finally in sections 5 and 6, we will discuss the related work and  possible future direction for this work.
%\note{Get back to this after we every other section is done}

\begin{figure}%[tpb!]
\centering
\includegraphics[width=\columnwidth]{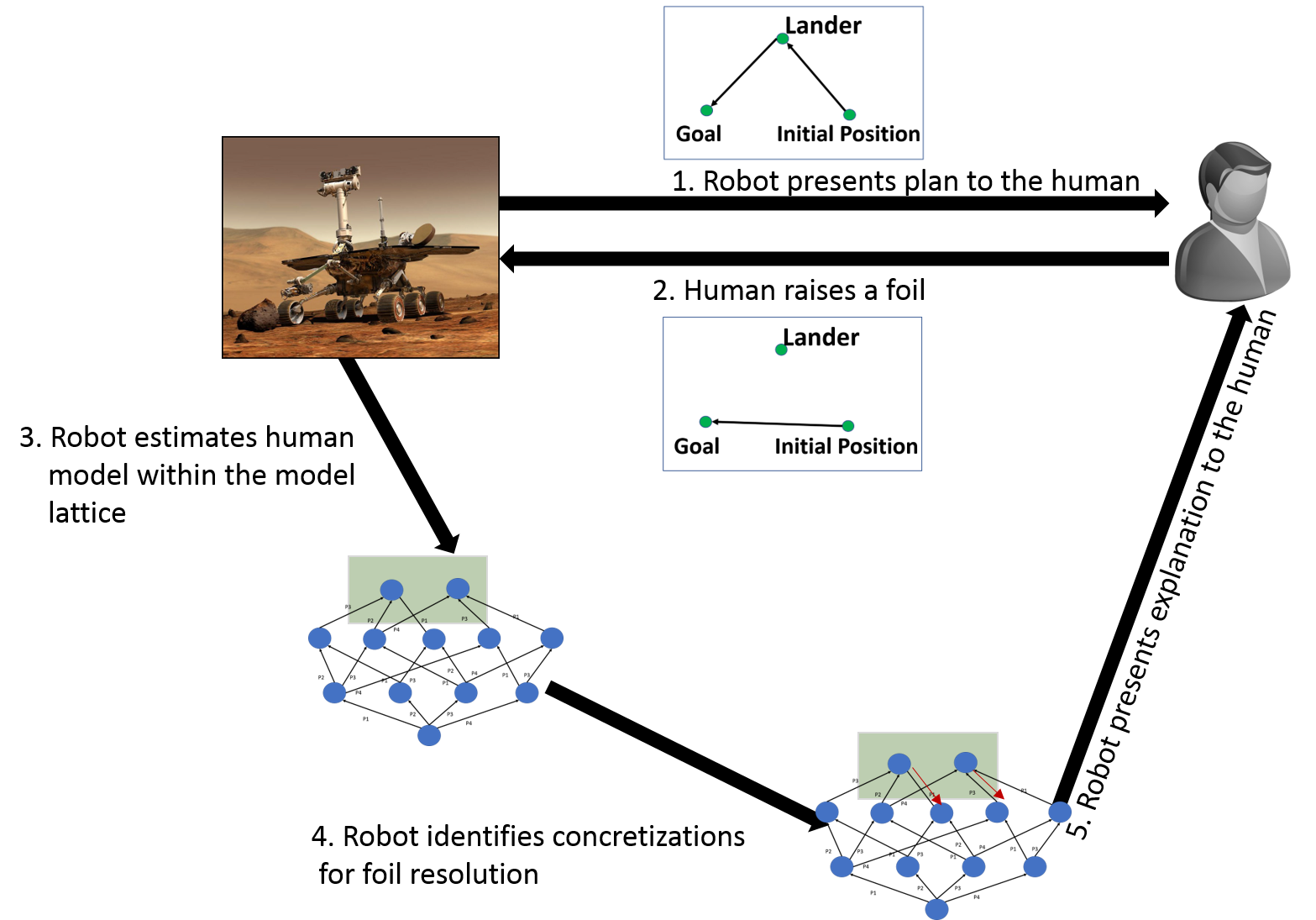}
\caption{{\small
An illustration of the hierarchical explanation process. The human observer who views the task at a higher level of abstraction expects the rover to execute a different plan from the one chosen by the rover. The rover presents the human with an explanation it believes will help resolve the foils in the human's updated model.}}
\label{img}
\end{figure}

\section{Hierarchical Expertise-Level Models}
In this work, we will focus on abstractions that form models by projecting out state fluents. While the presentation in the following sections is equally valid for both predicate and propositional abstractions, we will focus on propositional abstractions to keep our formulation clear and concise. We will look at planning models of the form $\mathcal{M} = \langle P, S, A, I, G\rangle$ where P gives the set of state fluents, S the set of possible states, $A$ the set of actions, $I$ the initial state and $G$ the goal. Each state $s \in S$ is uniquely represented by the set of propositions that are true in that state, i.e, $s \subseteq P$.

Each action $a\in A$ is associated with a set of preconditions $\textrm{prec}_{a}$ that needs to hold for the effects ($e_{a}$) of that action to be applied to a particular state. Each effect set $e_a$ can be further separated into a set of add effects $e_{a}^{+}$ and delete effects $e_{a}^{-}$. The result of executing an actions $a$ on a state $s$ in this setting is defined as follows 
\begin{align*}
a(S) = 
\begin{cases}
      (S \cup e_{a}^{+}) \setminus e_{a}^{-}, & \text{if}\ \textrm{prec}_{a}\subseteq  S\\
      S & \text{otherwise}
    \end{cases}
\end{align*}
A plan $\pi$ is defined as a sequence of actions ($\langle a_1,..,a_n\rangle$, $n$ being the size of the plan), and a plan is said to be executable for the problem $\mathcal{M}$ (i.e, $\pi(I) \models_{\mathcal{M}} G$) if $\pi(I) \supseteq G$.

Following works like \cite{seipp2013counterexample,backstrom2013bridging}, we will also use the concept of a transition system induced by the planning model to describe ideas related to abstraction. Intuitively, a transition system constitutes a graph where the nodes represent possible states, and the edges capture the transitions between the states that are valid in the corresponding planning model. We refer the readers to the previously mentioned works for further analyses of state transition systems and their connection to abstractions.
% Each planning model $\mathcal{M}$ induces a transition system $\mathcal{T}$ that captures the set of state transitions that are possible within this model. Intuitively, a transition system constitutes a graph where the nodes represent possible states, and the edges capture the transitions between the states that are valid in the corresponding planning model. We will use this concept of transition systems to define concepts related to abstraction. Transition systems provides us with a useful tool to study abstractions and has been widely used in abstraction heuristics literature \cite{seipp2013counterexample,backstrom2013bridging}. We refer the reader to these original works for further analyses of state transition systems and their connection to abstractions.%such abstractions and their other applications.% and refer readers to look at \cite{seipp2013counterexample} or \cite{backstrom2013bridging} for more information on transition systems for planning problems.

\begin{defn}
{\em %\textbf{Propositional Abstractions}\\
 For a set of states $S$, a set $X$ is said to be a \textbf{propositional abstraction} of $S$ with respect to some set of propositions $\Lambda$, if there exist a surjective mapping $f_{\Lambda}: S \rightarrow X$, such that for every state $s \in S$, there exists a state $f_{\Lambda}(s) \in X$ where $ f_{\Lambda}(s) = s \setminus \Lambda$
}
\end{defn}
For notational convenience we will refer to the set of states obtained by abstracting out the proposition set $\Lambda$ from some set of states $S$ as $[S]_{f_{\Lambda}}$.% and in this scenario we will refer to $S$ as the set of concrete states.
\begin{defn}
\label{model_abs}
{\em %\textbf{Abstract Model}\\
For a planning model $\mathcal{M} = \langle P, S, A, I, G\rangle$ with a corresponding transition system $\mathcal{T}$, a model $\mathcal{M}' = \langle P', S', A', I', G'\rangle$ with a transition system $\mathcal{T}'$ is considered an \textbf{abstraction of $\pmb{\mathcal{M}}$}, if there exist a set of propositions $\Lambda$, such that $P' = P - \Lambda$, $S' = [S]_{f_{\Lambda}}$, $I' = f_{\Lambda}(I)$, $G' = f_{\Lambda}(G)$ and for every transition $s_1 \xrightarrow{a} s_2$ in $\mathcal{T}$ corresponding to an action $a$, there exists an equivalent transition  $(s_1 \setminus \Lambda) \xrightarrow{a} (s_2 \setminus \Lambda)$ in $\mathcal{T}'$. \\
}
\end{defn}
As per Definition \ref{model_abs}, abstractions induce an ``imprecise" model of the underlying domain. All plans that were valid in the original model will have an equivalent plan in this new model. We will use the operator $\sqsubset$ to capture the fact that a model $\mathcal{M}$ is less abstract than the model $\mathcal{M}'$, i.e if $\mathcal{M} \sqsubset \mathcal{M}'$ then there exist a set of propositions $\Lambda$ such that $\mathcal{M}' = [\mathcal{M}]_{f_{\Lambda}}$.
% For a given model, if we restrict the set of possible propositions that can be projected out to some set $\Lambda_{\mathcal{P}}$, then we can arrange the set of all possible abstract models that can be formed by dropping some subset of $\Lambda_{\mathcal{P}}$ into a lattice. 
With the definition of abstraction and related notations in place, we are now ready to define a model lattice. Most approaches discussed in this paper will rely on this lattice to both estimate human's model and to identify explanations.
\begin{defn}
{\em %\textbf{Model lattice ($\pmb{\mathcal{L}}$)}\\
For a model $\mathcal{M}^{\#}$, the \textbf{model lattice} $\pmb{\mathcal{L}}$ is a tuple of the form $\mathcal{L} = \langle \mathbb{M}, \mathbb{E}, \mathbb{P}, \ell \rangle$, where $\mathbb{M}$ is the set of lattice nodes ,$\mathbb{E}$ the lattice edges, $\mathbb{P}$ the superset of propositions considered for abstraction within this lattice and finally $\ell$ is a function mapping edges to labels, provided $\mathcal{M}^{\#}\in\mathbb{M}$ and all nodes $\mathcal{M}' \in \mathbb{M}$ satisfy the condition $\mathcal{M}^{\#} \sqsubseteq \mathcal{M}'$. Additionally, for each edge $e_i = (\mathcal{M}_i, \mathcal{M}_j)$ there exists a proposition $p \in \mathbb{P}$ such that $[\mathcal{M}_i]_{f_{p}} = \mathcal{M}_j$ and $\ell(\mathcal{M}_i, \mathcal{M}_j) = p$.
}
\end{defn}
Thus each edge in this lattice corresponds to an abstraction formed by projecting out a single proposition (represented by the label of the edge). We can also define a concretization function $\gamma_{p}$ that retrieves the model that was used to generate the given abstract model by projecting out the proposition $p$, i.e, $\gamma_{p}(\mathcal{M}) = \mathcal{M}'$ if $(\mathcal{M}',\mathcal{M}) \in \mathbb{E}$ and $\ell(\mathcal{M}',\mathcal{M})=p$.

Through the rest of this work, we will make some assumptions on the structure of the lattice $\mathcal{L}$ and the abstraction methods used by $\mathcal{L}$ to simplify our discussions. In this paper, we will focus on lattices where each node in $\mathbb{M}$ has an incoming edge for every proposition missing from its corresponding model. We will refer to lattices that satisfy this property as \textbf{\lattype}~ lattices. Additionally, we will call a proposition conserving lattice that contains an abstract node corresponding to each possible subset of $\mathbb{P}$ as the \textbf{Complete Lattice} for $\mathcal{M}$ given $\mathbb{P}$.

A lattice $\mathcal{L}$ is \lattype, if for any model $\mathcal{M}\in \mathbb{M}$ and $\forall p\in \mathbb{P}$, if $p$ not in $P_{\mathcal{M}}$ then there exists a model $\mathcal{M}'\in\mathbb{M}$, such that $(\mathcal{M}',\mathcal{M})\in\mathbb{E}$ and $\ell(\mathcal{M}',\mathcal{M})=p)$.
Notice that enforcing conservation of propositions doesn't require any further assumptions about the human model and can be easily ensured by the agent generating the lattice.

We also assume that all abstraction functions used in generating the models in the lattice are commutative and idempotent, i.e., $[[\mathcal{M}]_{f_{p_1}}]_{f_{p_2}} =  [[\mathcal{M}]_{f_{p_2}}]_{f_{p_1}}$ and $[[\mathcal{M}]_{f_{p_1}}]_{f_{p_1}} =  [\mathcal{M}]_{f_{p_1}}$.
%One way to form a more abstract model would be to remove the proposition from the problem definition and making the effects non-deterministic if the proposition appeared in the precondition. 
Readers can refer to \cite{srivastava2016metaphysics} for a  comprehensive list of ways to generate imprecise abstract models that satisfy these properties. 
% We will also use a concretization function of the form $\gamma_{p}$ to retrieve the more concrete model which was used to generate the abstract model by projecting out the proposition $p$ (such that $\gamma_{p}([\mathcal{M}]_{f_{p}}) = \mathcal{M}$). 

As mentioned earlier, we consider an explanation generation setting where the human observer uses a task model (denoted as $\mathcal{M}_{H} = \langle P_{H}, S_{H}, A_{H}, I_{H}, G_{H}\rangle$), that is a more abstract version of the robot's model ($\mathcal{M}_{R} = \langle P_{H}, S_{R}, A_{R}, I_{R}, G_{R}\rangle$). While the robot may not have access to $\mathcal{M}_H$, it understands that $\mathcal{M}_H$ is a member of the set $\mathbb{M}$ for the lattice $\mathcal{L}$. The human comes up with the \textbf{foil set F } $= \{ \pi_1,\pi_2,...,\pi_m\}$ that the robot needs to refute by providing some explanation $E$ regarding the task. The explanation should contain information about specific domain properties (i.e., state fluents) that are missing from the human's model and how these properties affect different actions (for example which actions use these propositions as preconditions and which ones generate/delete them).% and thereby concretizing the human model. 
 We can represent such an explanation using the set of propositions whose concretization is required to refute the given foils.
\begin{defn}
{\em %\textbf{Explanation (E)}\\
An \textbf{explanation E }of size $n$ for the human model $\mathcal{M}_H$ and a foil set $F$ can be represented as a set of propositions of the form $E=\{ p_1,..., p_n\}$ such that \\
$\forall \pi \in F, \pi(I_{\gamma_{E}(\mathcal{M}_H)}) \not\models_{\gamma_{E}(\mathcal{M}_H)} G_{\gamma_{E}(\mathcal{M}_H)}$\\
%\]
Where $\gamma_{E}(\mathcal{M}_H)$ is the model obtained by applying the concretizations corresponding to $E$ on the model $\mathcal{M}_H$ .
}
\end{defn}
%To better understand our setting, 
\begin{exmpl}
\label{rov_exmpl}
Consider a simplified version of the rover domain mentioned earlier. Suppose the rover uses a modified version of the IPC rover domain \cite{ipc} that also takes into account the battery level of the robot. Each rover operation has a different energy requirement, and the battery level needs to be above a predefined threshold for it to execute them, e.g., it can perform rock sampling only if the battery level is above 75\%. Furthermore, the rover needs to visit the base station (i.e., the lander) and perform a reset action to recharge its batteries.

The rover knows that the human observer is at most ignorant of its energy requirements and/or storage capabilities. So the model lattice $\mathcal{L}$ needs to consider abstractions corresponding to the following propositions {\small $\mathbb{P}$=\{\textsf{battery\_level\_above\_25\_perc}, \textsf{battery\_level\_above\_50\_perc},  \textsf{battery\_level\_above\_75\_perc}, \textsf{full\_store}$\}$.}
Figure \ref{lattice} shows the lattice that the robot would use in this setting. Here we will create each abstract model by dropping a proposition from the more concrete model and by making the effects of action non-deterministic if the dropped predicate appears in the precondition. For example, if the action {\small \textsf{drop\_store1}} has effects of the form
{\small\begin{align*}
\{\textsf{full\_store1},~\textsf{store\_of\_store1}\}\rightarrow\{\neg\textsf{full\_store1},~\textsf{empty\_store1}\}
\end{align*}}
Now in an abstract version of this model, if the proposition {\small\textsf{full\_store1}} is dropped the effect becomes
{\small\begin{align*}
\{~\textsf{store\_of\_store1}\}\rightarrow~ND\{~\textsf{empty\_store1}\}
\end{align*}}
Which now says that the action's effects are non-deterministic  and executing {\small \textsf{drop\_store1}} may or may not turn the fluent {\small \textsf{empty\_store1}} true.\\
Let the plan $\pi_R$ be $\langle{\small\textsf{
navigate\_w0\_lander,
reset\_at\_lander,}}$ ${\small\textsf{
navigate\_lander\_w1,
sample\_rock\_store0\_w1}}\rangle$
and the foil set $F$ be $\{\langle
{\small\textsf{
navigate\_w0\_w1, navigate\_lander\_w1,}}$ ${\small\textsf{
sample\_rock\_store0\_w1
}}\rangle\}$
\end{exmpl}
In Example \ref{rov_exmpl}, the rover would have difficulty coming up with a single explanation as it does not know $\mathcal{M}_H$. One possibility would be to restrict its attention to just the models that are consistent with the foils 
%Possible defn for belief space over human model%
. In this scenario, this would correspond to $\{c6, c7, c9, c10, c11, c12\}$.
\begin{figure}[tbp]
\centering
\includegraphics[width=0.70\columnwidth]{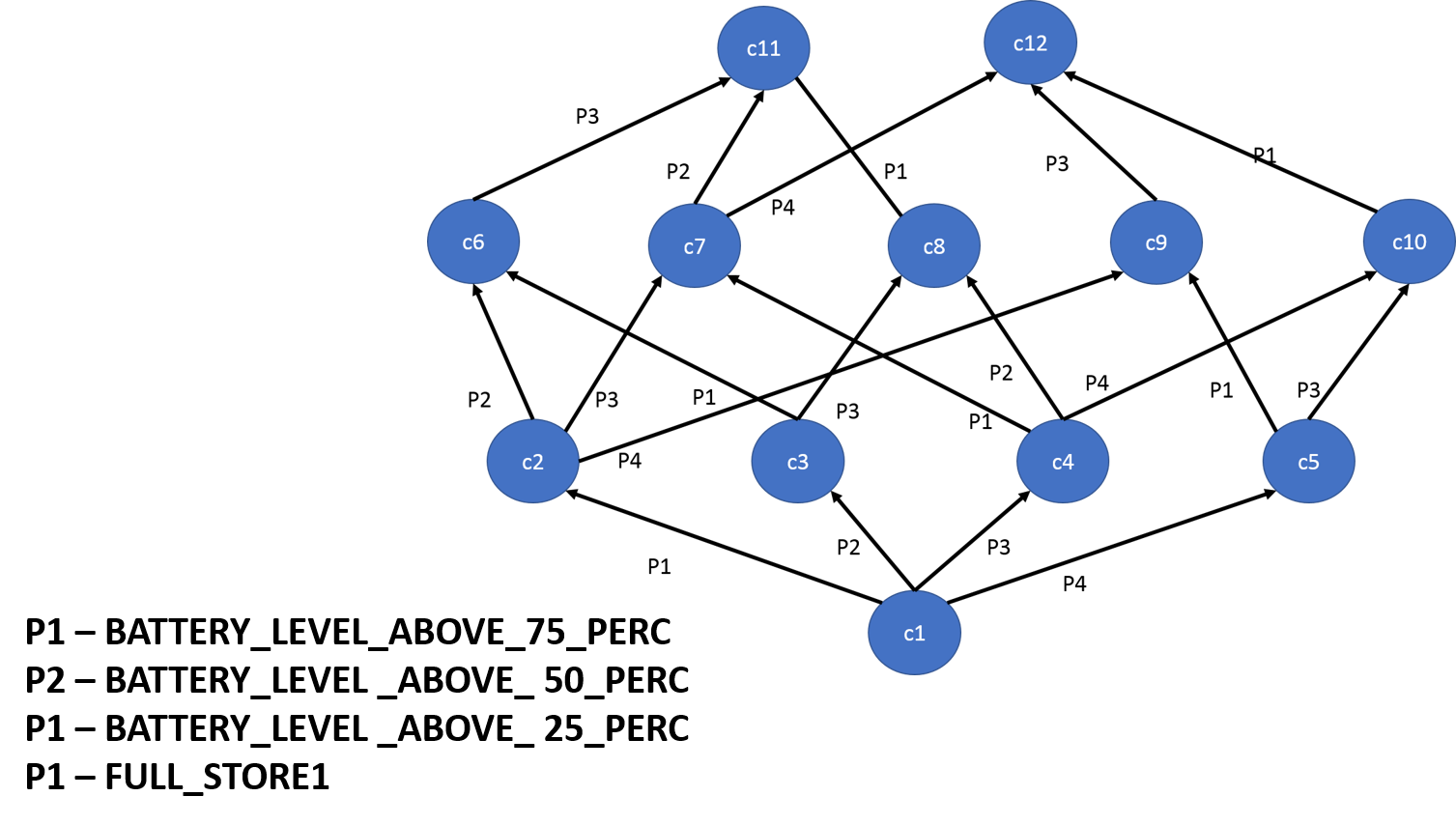}
\caption{{\small A possible abstraction lattice for the  grounded rover domain.}}
\label{lattice}
\end{figure}
Now we need to find a way of generating explanations given this reduced set of models. 
%As we will see, we can generate valid explanation using a set of models.
% In most domains, while the foils may help decrease our uncertainty about the observer's model, it is unlikely that we will be able to identify the exact model just from these observations. But as we will see, we do not need to know the exact human model.
\begin{prop}
\label{prop_abstract}
Let $\mathcal{M}_{i}$ be some model in $\mathcal{L}$ such that $\mathcal{M}_{H} \sqsubseteq \mathcal{M}_{i}$. If $E$ is a valid explanation for $\mathcal{M}_{i}$ and some foil set $F$, then $E$ must also explain $F$ for $\mathcal{M}_{H}$.
\end{prop}
This proposition directly follows from the fact that for a proposition conserving lattice $\gamma_{E}(\mathcal{M}_{i})$ will be a logical weaker model than $\gamma_{E}(\mathcal{M}_{H})$. 

Next, we will define the concept of a minimal abstraction set for a given lattice $\mathcal{L}$ and foils $F$
\begin{defn}
{\em %\textbf{Minimal Abstraction Set} $\pmb{\mathbb{M}_{min}}$\\
Given an the abstraction lattice $\mathcal{L} = \langle\mathbb{M},\mathbb{E},\mathbb{P},\ell\rangle$ the {\em\textbf{minimal abstraction set $\mathbb{M}_{min}$}} is the supremum of all the models that are consistent with the foil set $F$.\\
$\mathbb{M}_{min} = \sup\{\mathcal{M}_i|\mathcal{M}_i \in \mathbb{M}, \forall \pi \in F (\pi(I_{\mathcal{M}_i}) \models_{\mathcal{M}_i} G_{\mathcal{M}_i}) \}$
}
\end{defn}
In Example \ref{rov_exmpl}, the minimal abstract model set will be $\mathbb{M}_{min} = \{c11, c12\}$.

If we can find an explanation that is valid for all the models in $\mathbb{M}_{min}$ then by Proposition \ref{prop_abstract} it must work for $\mathbb{M}_{H}$ as well. 
\begin{prop}
\label{prop_existence}
For a given model lattice $\mathcal{L}$, the minimal abstraction set $\mathbb{M}_{min}$ and a set of foils $F$, there exist an explanation $E$ such that for each $\mathcal{M}' \in \mathbb{M}_{min}$ and $\forall \pi \in F~$, \\%the following condition holds\\ 
$\pi(I_{\gamma_{E}(\mathcal{M}')}) \not\models_{\gamma_{E}(\mathcal{M}')} G_{\gamma_{E}(\mathcal{M}')}$
\end{prop}
It is easy to see why this property holds, as any explanation that involves concretizing all possible propositions in $\mathbb{P}$ satisfies this property. 

In most cases, we would prefer to get the least costly or the shortest explanation (if all concretizations are equally expensive) to the listener. In the rover example, even if the human is unaware of multiple task details, the robot can easily resolve the user's doubts by explaining the concretizations related to the proposition {\small\textsf{battery\_level\_above\_75\_perc}} without getting into other details. Describing the details of others propositions is unnecessary and in the worst case might leave the human feeling overwhelmed and confused. In this case, the explanation would just include
% For example, in the rover domain even if the human were unaware of many task details, the robot can easily resolve the users doubts about the plan by providing information about the proposition \texttt{BATTERY\_LEVEL\_75\_PERC\_BT1} and how the robot only performs the rock sampling operation if its battery is at 75\%. The actual contents of this explanation would look something like this
Information regarding battery levels and 
how to identify when the battery level 
is or above 75\% and
model updates like \\
{\small\textsf{
sample\_rock-has-precondition-battery\_level\_above\_75\_perc\\
sample\_soil-has-precondition-battery\_level\_above\_75\_perc\\
...
}}\\
% It would be unnecessary and distracting for the robot to provide details about all possible missing propositions instead of finding the minimal set.
Before delving into the optimization version, let us look at the complexity of the corresponding decision problem
\begin{theorem}
Given a minimal abstraction set $\mathbb{M}_{min}$, a plan $\pi_R$, the set of propositions being abstracted $\mathbb{P}$ and the set of foils $F$ for a model $\mathcal{M}$, the problem of identifying whether an explanation of size $k$ exists for the complete lattice is \textbf{NP-complete}.
\end{theorem}
\begin{proof}[Proof (Sketch)]
The fact that we can test the validity of the given explanation in polynomial time (size of the explanation is guaranteed to be smaller than $|\mathbb{P}|$) shows that the problem is \textbf{NP}. 
We can show \textbf{NP-completeness} by reducing the set covering problem \cite{bernhard2008combinatorial} to an instance of the explanation generation problem. Let's consider a set covering problem with $U$ as the universe set and $S$ as the set of sub-collections. Now let us create an explanation generation problem where the set of foils $F$ is equal to $U$ and the propositions in the set $\mathbb{P}$ contains a proposition for each member of $S$. Additionally concretizing with respect to a proposition will resolve only the foils covered by its corresponding subset in $S$. For this setting, we can construct a fully connected proposition conserving lattice $\mathcal{L}$ of height $|S|$. Within the lattice, there exists a unique most abstract model where all the foils hold and a single most concrete model (where none of the foils hold). Now if we can come up with an explanation of size k in this setting, then this explanation corresponds to a set cover of size k.
\end{proof}

%\section{Algorithms to Generate Minimal Explanations}
\section{Generating Minimal Explanations}
As mentioned earlier, we are interested in producing the smallest possible explanation. Additionally, in most domains, the cost of communicating the concretization details could vary among propositions. An explanation that involves a proposition that appears in every action definition might be harder to communicate than one that only uses a proposition that is part of the definition of a single action.

In addition to the actual size, the comprehensibility of the explanations may also depend on factors like human's mental load, the familiarity with the concepts captured by the propositions, etc.. To keep our discussions simple, we will restrict the cost of communicating an explanation to the number of unique model updates this explanation would bring about in the human model.%\footnote{Note, that this cost is independent of the actual model that the human actually holds. Adding the same propositions on two different models would involve applying the same operations (you add the propositions as preconditions for certain actions and add it to the effects list of other actions), but they may induce different models based on the exact model.}
 We will use the symbol $C_{p}$ to represent the cost of communicating the changes related to the proposition $p$. We will overload $C$ to also work on sets of propositions.

Now our problem is to find the cheapest explanation (represented as $E_{min}$) for a given set of foils $F$, and the minimal abstract model set $\mathbb{M}_{min}$. One possibility is to perform an A* search \cite{hart1968formal} %(\note{Needs to be changed}) 
over the space of possible propositional concretizations to identify $E_{min}$. Each search state consists of the minimal set of abstract models for the human model given the current explanation prefix. We will stop the search as soon as we find a state where the foils no longer hold for the current minimal set. 
\begin{prop}
Let $\mathbb{M}_{min}$ be the minimal abstraction set for a given lattice $\mathcal{L}=\langle \mathbb{M}, \mathbb{E}, \mathbb{P},\ell\rangle$ and foil set $F$. Then for a proposition $p$, the set $\widehat{\mathbb{M}}_{min}$ formed by applying the concretization corresponding to $p$ on every element of $\mathbb{M}_{min}$ will be the minimal abstract set for $\widehat{\mathbb{M}}$ formed by applying the concretization $\gamma_{p}$ on every element of $\mathbb{M}$ given $F$.
\end{prop} 
The above property implies that we don't need to look at the lattice $\mathcal{L}$ to recalculate minimal abstraction set after the application of every concretization function. We can also further simplify our problem by exploiting the fact that a particular propositional concretization resolves a foil (i.e., make the foil no longer valid) when it either adds a precondition (or a new condition for a conditional effect) or a goal fact that can not be satisfied by the foil. To concisely capture this idea we will introduce the concept of a foil resolution set to represent the subset of foils resolved by the concretization of a particular proposition. 
\begin{defn}
{\em %\textbf{Resolution Set }$\pmb{\mathcal{R}_{F}}$\\
For a set of models $\mathbb{M}'$, a foil set $F$ and a proposition p, the \textbf{resolution set} $\mathcal{R}_{F}(\mathbb{M}',p)$ gives the subset of foils that no longer holds in the concretized models, i.e $\mathcal{R}_{F}(\mathbb{M}',p) = \{\pi| \pi \in F \wedge (\forall\mathcal{M}'\in \mathbb{M}'(\pi(I_{\gamma_{p}(\mathcal{M}')}) \not\models_{\gamma_{p}(\mathcal{M}')}~G_{\gamma_{p}(\mathcal{M}')}))\}$. 
}
\end{defn}
We will also use $\mathcal{R}_{F}$ to represent the set of foils resolved by a sequence of propositions
\begin{prop}
\label{prop5}
For a set of  model $\mathbb{M}'$ and a foil set $F$
\[\mathcal{R}_{F}(\mathcal{M}',\langle p_1,p_2\rangle) = \mathcal{R}_{F}(\mathcal{M}',\langle p_1\rangle)~\cup~  \mathcal{R}_{F}(\mathcal{M}',\langle p_2\rangle)\]
\end{prop}
The above property implies that concretizing any $n$ propositions cannot resolve foils that weren't resolved by the individual propositions.
\begin{prop}
\label{prop6}
For two models $\mathcal{M}_1$, $\mathcal{M}_2$ and a set of foils $F$, if $\mathcal{M}_1\sqsubseteq\mathcal{M}_2$ then for any proposition $p$, $\mathcal{R}_{F}(\{\mathcal{M}_1\}, p) \subseteq \mathcal{R}_{F}(\{\mathcal{M}_2\},p)$
\end{prop}
The above proposition ensures that if an explanation is the minimal one for $\mathbb{M}_{min}$, then it must be the minimal explanation for $\mathcal{M}_{H}$ as well.

These propositions will be instrumental in proving the effectiveness of our greedy algorithm described by Algorithm \ref{algo_greedy}. In each iteration of this search, the algorithm greedily chooses the proposition that minimizes $\frac{C_{p}}{|F' \cap \mathcal{R}_{F}(\mathbb{M}', p)|}$, where $F'$ is the set of unresolved foils at that iteration and the search ends when all foils are resolved.
\begin{algorithm}[tbp!]
\scriptsize
\caption{Greedy Algorithm for Generating $\widehat{E}$}
\label{algo_greedy}
\begin{algorithmic}[1]
\Procedure{Greedy-exp-search}{}
\vspace{2pt} 
\BState \emph{Input}:~~~~$\langle F, \mathcal{L}=\langle \mathbb{M}, \mathbb{E}, \mathbb{P}, \ell \rangle\rangle$
\BState \emph{Output}: Explanation $\widehat{E}$
\vspace{2pt} 
\BState \emph{Procedure}:  
\State curr\_model = $\langle\mathbb{M}_{min},F\rangle$
\State $\widehat{E} = \{\}$
\State $\mathbb{M}_{min}$~~~~~~~~~~~~$\leftarrow $MinimalAbstractModels($\mathcal{L}, F$)
\State Precompute the resolution sets $\mathcal{R}_{F}(\mathbb{M}_{min}, p)$ for each $p \in \mathbb{P}$
\vspace{2pt}
\While{True}
\vspace{2pt} 
\State $\mathbb{M}',F' = \textrm{curr\_model}$
\If{$|F'| = 0$} return $\widehat{E}$  \Comment{\textcolor{black}{Return $\widehat{E}$ if all the foils are resolved}}
\Else
\State $p_{next} = \argmin_{p}(\frac{C_{p}}{|F' \cap \mathcal{R}_{F}(\mathbb{M}', p)|})$
\State $\mathbb{M}_{new} = \{\gamma_{p_{next}}(\mathcal{M})|\mathcal{M}\in \mathbb{M}'\}$
\State curr\_model = $\langle\mathbb{M}_{new},F\setminus\mathcal{R}_{F}(\mathbb{M}', p)\rangle$
\State $\widehat{E} = \widehat{E} \cup p$
\EndIf
\EndWhile
\EndProcedure
\vspace{4pt} 
\end{algorithmic}
\end{algorithm}
\begin{theorem}
The explanation $\widehat{E}$ generated by Algorithm \ref{algo_greedy} for a set of foils $F$ and a lattice $\mathcal{L}=\langle \mathbb{M}, \mathbb{E}, \mathbb{P}, \ell \rangle$ is less than or equal to $(\ln k)*C_{E_{min}}$, where $C_{E_{min}}$ is the cost of an optimal explanation and $k$ represents the maximum number of foils that can be resolved by concretizing a single proposition, i.e, $k = \max_{p}|\mathcal{R}_{F}(\mathbb{M}_{min}, p)|$.
\end{theorem}
\begin{proof} [Proof (Sketch)]
We will prove the above theorem by showing that Algorithm \ref{algo_greedy} corresponds to the greedy search algorithm for a weighted set cover problem. Consider a weighted set cover problem $\langle U, S, W\rangle$ such that the universe set $U = F$, the subcollections set S is defined as $S = \{s_p| p \in \mathbb{P}\}$ where $s_p = \mathcal{R}_{F}(\mathbb{M}_{min},p) $ and the cost of each subset $s_p$ is gives as $W(s_p) = C_{p}$. Proposition \ref{prop5} ensures that the size of resolution set is a submodular and monotonic function. In this setting, the act of identifying a set of propositions that resolve the foil set is identical to coming up with a set cover for $U$ in the new weighted set cover problem. Furthermore, we can show that the optimal set cover $\mathcal{C}_{opt}$ must correspond to the cheapest explanation $E_{min}$ (We can prove this equivalence using Propositions \ref{prop_abstract},\ref{prop_existence} and \ref{prop5}, we are skipping the details of this proof due to space constraints). Algorithm \ref{algo_greedy} describes a greedy way of identifying the cheapest set cover for this weighted set cover problem and thus the minimal explanation for the original problem. For weighted set cover the above greedy algorithm is guaranteed to generate solutions that are at most $\ln k *W(\mathcal{C}_{opt})$ \cite{young2008greedy}, where $k = \max_{s\in S}|s|$  and this approximation guarantee will hold for $E_{min}$ as well.
\end{proof}
We will use this algorithm to both generate solutions and to calculate an inadmissible heuristic for the previously mentioned A* search. For the heuristic generation, we will further simplify the calculations (specifically step 8 in Algorithm \ref{algo_greedy}) by considering an over-approximation of $\mathcal{R}_{F}$. Instead of considering the set of all foils resolved by concretizing each proposition $p$, we will consider the set of foils where $p$ appears in the precondition of one of the actions in it. This set should be a superset for $\mathcal{R}_{F}$ for any proposition.  
\begin{figure}[tbp]
\centering
\includegraphics[width=\columnwidth]{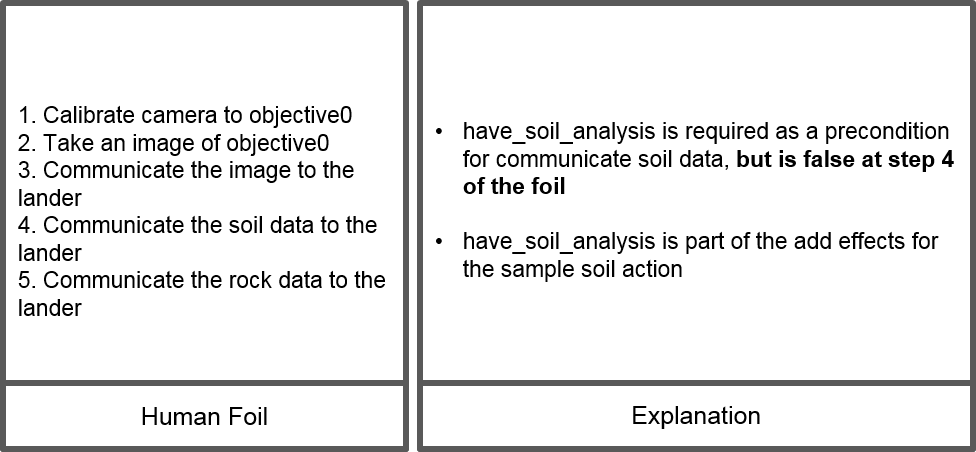}
\caption{{\small An example explanation generated by our system for IPC rover domain. The human incorrectly believes that the rover can communicate sample information without explicitly collecting any samples. While the abstraction lattice in this example was generated by projecting out upto 12 predicates, the search correctly identifies concretizations related to {\small\textit{(have\_soil\_analysis ?r - rover ?w - waypoint)}} as the cheapest explanation ($C_{E} = 2$ as opposed to $C_{\mathbb{P}} = 55$)}}
\label{expl_example}
\end{figure}
\section{Empirical Evaluations}
In our evaluation, we wanted to understand how effective our approaches were in terms of the conciseness of the explanations produced, the solution generation time and the usefulness of approximation. For the approximation, we were interested in identifying the trade-off between decrease in runtime vs. reduction in solution quality.

All three explanation methods discussed in this paper (Blind, heuristic and greedy) were evaluated on five IPC benchmark domains\cite{ipc}. All the experiments detailed in this section were run on an Ubuntu workstation with 12 core Intel(R) Xeon(R) CPU and 64G RAM. %The implementation of the search algorithms along with the test benchmarks will be made available to the public after the double-blind review period.

For each domain, we selected 30 problems from either available test sets or by using standard problem generators (the problems sizes were selected to reflect the size of previous IPC test problems). The lattice for each problem domain pair was generated by randomly selecting 50\% of domain predicates and then generating a fully connected proposition conserving lattice using that set of predicates. Each abstract model was created using ND-operators similar to Example \ref{rov_exmpl}. Each search generates the set of proposition whose concretizations can resolve the foils set $F$. In actual applications, this set of propositions needs to be converted into an explanan (the actual message) by considering how this proposition is used in the robot model. Figure \ref{expl_example} shows the explanation generated by our approach for a problem in Rover domain.

\begin{figure*}
\tiny
  \begin{tabular}{|l|c|c|c|c|c|c|c|c|c|c|c|c|}
    \hline
    \multirow{1}{*}{Domain Name} &
    \multirow{1}{*}{$|\mathbb{P}|$} &\multirow{1}{*}{$C_{\mathbb{P}}$}&\multirow{1}{*}{$|F|$} &
%      \multicolumn{3}{c|} {}&
%      \multicolumn{3}{c|} {}\\
%       &
%       & &
       \multicolumn{3}{c|}{Blind Search (Optimal)}
      & \multicolumn{3}{c|}{Heuristic Search}& \multicolumn{3}{c|}{Greedy Set Cover} \\
%       %\\[1ex]
%       & &
%       &
%       & \multicolumn{3}{c|}{}
%       & \multicolumn{3}{c|}{}
%       & \multicolumn{3}{c|}{}
      %\\[1ex]
      &
      &&
	      &Cost & Size & Time(S) &Cost&Size& Time(S) & Cost&Size & Time(S)  \\
      \hline
      %\midrule
      \multirow{4}{*}{Barman} 
      &84.07&7 &1&6.87&1&2.43&6.87&1&2.08&6.87&1&3.61\\
&84&7&2&8.94&1.22&6.35&8.94&1.22&5.71&9.90&1.39&6.05\\
&90.7&7&4&17.19&1.77&24.99&17.19&1.77&23.7&18.45&1.97&10.34\\

%       &84.07&7 &1&7.57&1&2.40&7.57&1&2.16&7.57&1&3.63   \\ 
%       &84&7&2&10.35&1.27&7.31&10.36&1.27&7.55&11.24&1.46&5.84\\
%       %&&&3&13.17&1.55&15.54&13.17&1.56&15.48&14.03&1.73&7.84\\
%       &90.7&7&4&17.33&1.81&26.29&17.34&1.81&24.73&18.27&1.96&10.36\\
      \hline
       \multirow{4}{*}{Rover} 
       &168.66&12&1&3.58&1&7.86&3.58&1&5.22&3.58&1&19.18\\
&188.83&12&2&6.13&1.48&51.36&6.12&1.48&34.04&6.26&1.52&30.5\\
&192.83&12&4&10.87&2&203.83&10.87&2&181.87&11.42&2.19&49.32\\
%        &168.66&12&1&4.42&1&3.00&4.42&1&3.37&4.42&1&20.46\\
% &188.83&12&2&6.39&1.52&10.33&6.39&1.52&10.28&6.58&1.61&40.68\\
% &192.83&12&4&10.74&1.97&58.18&10.74&1.96&60.72&11.351&2.23&56.75\\
       %&168.66&12&1&4.05&1&21.72&&&&4.05&1&17.65\\
% &188.83&12&2&5.39&1.42&23.50&&&&5.60&1.51&30.45\\
% %&&&3&7.59&1.64&167.51&&&&7.99&1.83&39.34\\
% &192.83&12&4&9.62&1.99&152.14&&&&10.28&2.25&49.08\\
       \hline
      \multirow{4}{*}{Satellite} 
      &53.01&4&1&18.73&1&2.23&18.73&1&1.92&18.73&1&1.49\\
&60.77&4&2&32&1.61&7.21&32&1.6&5.86&32.53&1.7&3.04\\
&62.73&4&4&43.27&2.29&18.67&43.27&2.29&16.42&43.88&2.39&5.85\\
%       &53.01&4&1&22.04&1&2.56&22.04&1&2.73&22.04&1&2.07\\
% &60.77&4&2&36.76&1.60&7.36&36.76&1.59&7.63&37.19&1.66&3.76\\
% &62.73&4&4&51.18&2.36&18.40&51.18&2.36&18.85&51.69&2.46&6.43\\
       \hline
%       \texttt{Satellite} & 4& 133.5&1&59.1&4.82&1&59.1&4.36\\
%        \midrule
      \multirow{4}{*}{Woodworking} 
      &156.71&7&1&14.45&1&2.84&14.45&1&2.23&14.45&1&3.35\\
&146.33&7&2&20.62&1.21&6.88&20.62&1.21&4.93&21.38&1.38&6.25\\
&154&7&4&28.62&1.69&24.70&28.62&1.69&19.49&30.41&2&12.13\\

%       &156.71&7&1&14&0.89&2.40&14&0.89&2.77&15.61&1&3.85\\
% &146.33&7&2&12.07&0.86&4.59&12.07&0.85&4.67&18.9&1.55&6.98\\
% %&&&3&18.28&1.24&10.47&16.33&1.21&9.85&26.61&1.94&9.38\\
% &154&7&4&17.67&1.18&15.85&17.68&1.18&14.57&32.87&2.44&12.19\\
      \hline
      \multirow{4}{*}{Sokoban} 
      &220.6&3&1&51.21&1&1.51&51.21&1&1.35&51.21&1&1.28\\
&151.72&3&2&94.52&1.55&3.93&94.52&1.55&3.35&98.31&1.73&2.59\\
&220.69&3&4&136.41&2.22&8.75&136.41&2.22&8.3&141.93&2.37&5.23\\
%       &220.6&3&1&65.48&1&1.24&65.48&1&1.35&65.48&1&1.24\\
% &151.72&3&2&98.19&1.62&3.64&98.19&1.62&3.55&102.14&1.81&2.34\\
% %&&&3&159.17&1.90&6.17&159.17&1.90&6.27&167.96&2.11&4.84\\
% &220.69&3&4&127.54&1.89&7.27&127.54&1.89&6.85&146.28&2.28&6.62\\
    \hline
    %\bottomrule
  \end{tabular}
  \raisebox{-0.5\height}{\includegraphics[width=.30\textwidth]{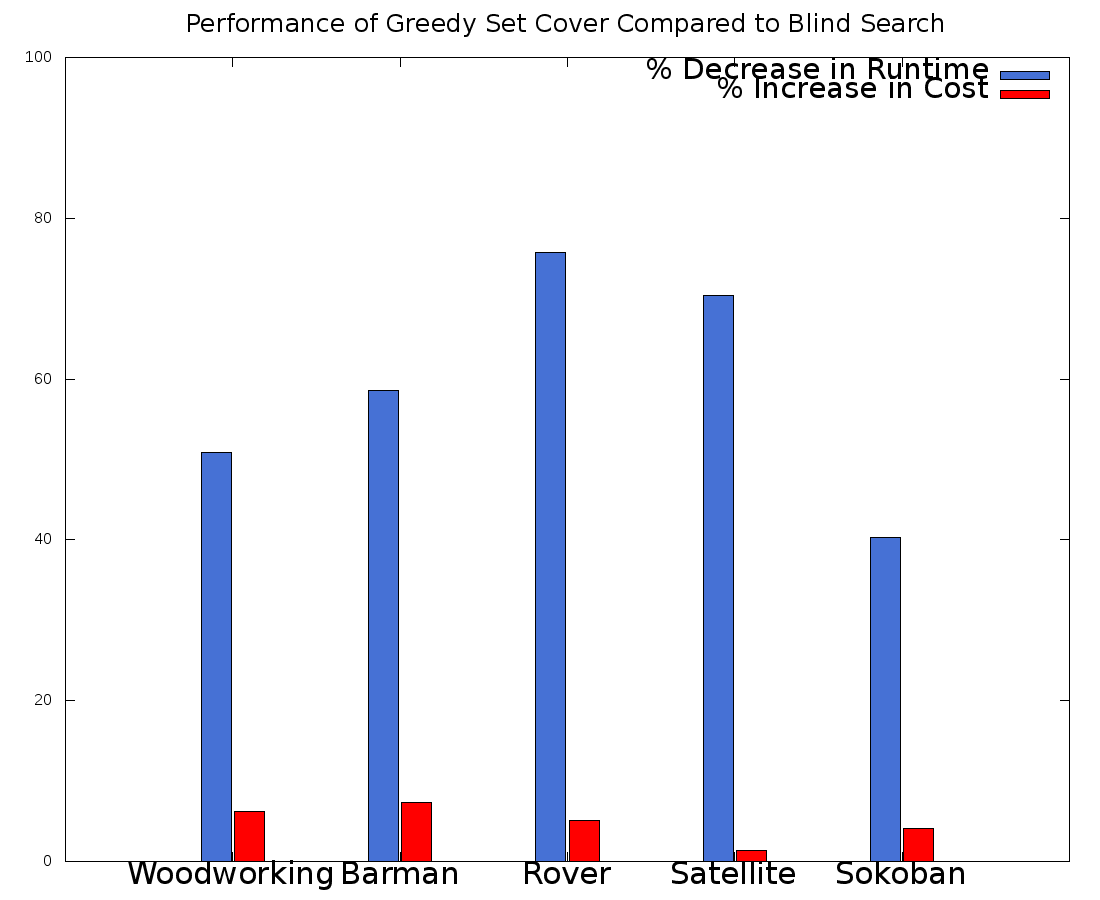}}
\caption{
{\small Table showing runtime/cost for explanations generated for standard IPC domains.Column $|\mathbb{P}|$ represents number of predicates that were used in generating the lattice, while $C_{\mathbb{P}}$ represents the cost of an explanation that tries to concretize all propositions in $\mathbb{P}$ and provides an upper bound on explanation cost. The graph on right side compares the performance of greedy set cover against the optimal blind search for $|F|=4$. It plots the average time saved by the set cover and the average increase in cost of the solution for each domain.}
}
\label{tab1}
\end{figure*}
% For this foil, our algorithm chose to explain the need for the robot to collect soil samples before communicating the sample information (i.e concretize with respect to the predicate {\small\textsf{(have\_soil\_analysis ?r - rover ?w - waypoint)}}), which was produced by our blind search in under 5 seconds. 
% Depending on the planning capabilities of the human, it is possible that the observer may pose additional foils after the robot provides the above explanation. But in most cases, we could come up with an explanation shorter than the set of all concretizations, whose application on human model could produce a model where either there exist no additional foils for this specific problem or the human cannot produce any further foils.
% Thereby saving the human the trouble of parsing an extremely expensive explanation.

The table in Figure \ref{tab1} presents the results from our empirical evaluation on the IPC domains. The table shows the average cost/size of each explanation along with the time taken to generate them. Note that by size, we refer to the no of predicates that are part of the explanation while the cost reflects the total number of unique model updates induced by that explanation. We attempted explanation generation for foil set sizes of one, two and four per problem.

The first point of interest is that the heuristic search seems to outperform blind search in almost every problem and generates near-optimal solutions (Blind search always generates the minimal explanation). Further, we saw that greedy search outperformed heuristic search in most cases barring a few exceptions. The greedy was able to make significant gains especially for higher foil sizes. This is entirely expected due to the fact that step 8 in Algorithm \ref{algo_greedy} can be expensive for problems with long plans (but still polynomial). This expensive pre-computation pays off as we move to cases where $E_{min}$ consists of multiple propositions. Additionally, we found out that greedy solutions were quite comparable to the optimal solutions with respect to their costs.
\section{Robot Demonstration}
\begin{figure}[tbp]
\centering
\includegraphics[width=\columnwidth]{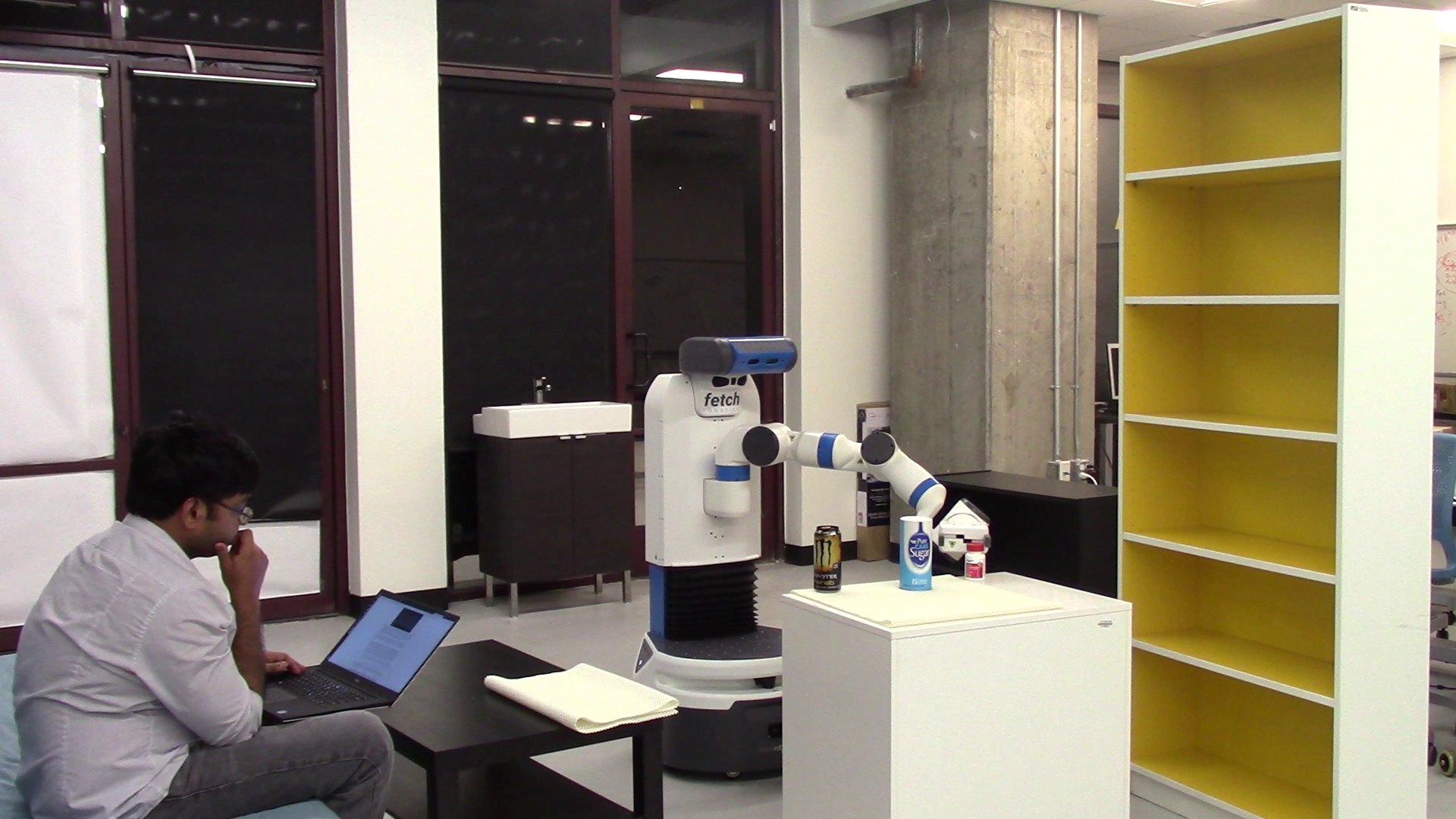}
\caption{{\small The grocery putaway domain setting.}}
\label{groc_robot}
\end{figure}
\begin{figure}[tbp]
\centering
\begin{tabular}{|l|}
\hline
\\ 
    (:action pickup\\
        :parameters (?x - item ?y - storage  ?u - pose ?v - traj)\\
        :precondition (and\\
\hspace{1cm}(is\_pickup\_pose ?u ?x)\\
\hspace{1cm}(is\_collision\_free\_traj ?x ?y ?u ?v)\\
\hspace{1cm}(in ?x ?y)\\
\hspace{1cm}(handempty)\\
        )\\
        :effect (and \\
\hspace{1cm}(not (handempty))\\
\hspace{1cm}(not (in ?x ?y))\\
\hspace{1cm}(holding ?x)\\
\hspace{1cm}(increase (total-cost) 1)\\
        )\\
    )\\
\\
    (:action place\_in\_high\_shelf\\
        :parameters (?x - item ?y - storage ?u - pose ?v - traj)\\
        :precondition (and\\
\hspace{1cm}(is\_putdown\_pose ?u ?x)\\
\hspace{1cm}(is\_collision\_free\_traj ?x ?y ?u ?v)\\
\hspace{1cm}(is\_condiment\_type ?x)\\
\hspace{1cm}(holding ?x)\\
\hspace{1cm}(is\_high\_shelf ?y)\\
        )\\
        :effect (and\\ 
\hspace{1cm}(handempty)\\
\hspace{1cm}(in ?x ?y)\\
\hspace{1cm}(not (holding ?x))\\
\hspace{1cm}(item\_putaway ?x)\\
\hspace{1cm}(increase (total-cost) 1)\\
        )\\
    )\\
\hline
\end{tabular}
\caption{{\small The action definitions for \textsf{pickup} and \textsf{place\_in\_high\_shelf} from the most concrete model.}}
\label{concr-act}
\end{figure}
%\note{
This section describes a demonstration of our approach on a physical robot for a simple grocery putaway task. Figure \ref{groc_robot} presents the basic setup for the task. The goal of the robot here is to put away a bottle of tablets, a can of energy drink and a jar of sugar to proper storage locations. The storage location of each object is decided based on its type, for example, the robot should place the medicine bottle in the medicine cabinet, the sugar jar in the high pantry shelf, while the energy drink needs to be handed over to the human. In addition to these task-level constraints, the robots operations are restricted by various motion level constraints that limit the possible physical movements that the robot can perform. For example, given the current position of the sugar jar on the table, the robot couldn't come up with any pickup pose that would allow the robot to place the sugar jar on the high shelf. In such cases, the robot could always enlist the help of the human to complete the plan.

In this setting, we will assume that the most concrete robot model consists of action descriptions that include both task-level symbols as well as continuous geometric arguments. Figure \ref{concr-act} presents the definitions for {\small \textsf{pickup}} and {\small \textsf{place\_in\_high\_shelf}} actions in the most concrete model. In this model, the arguments of type ?pose and ?traj represents the pickup/putdown pose (the position and orientation of the end effector) and motion plans followed by the robot to perform the pickup/putdown.%}
\begin{figure}[tbp]
\centering
\includegraphics[width=\columnwidth]{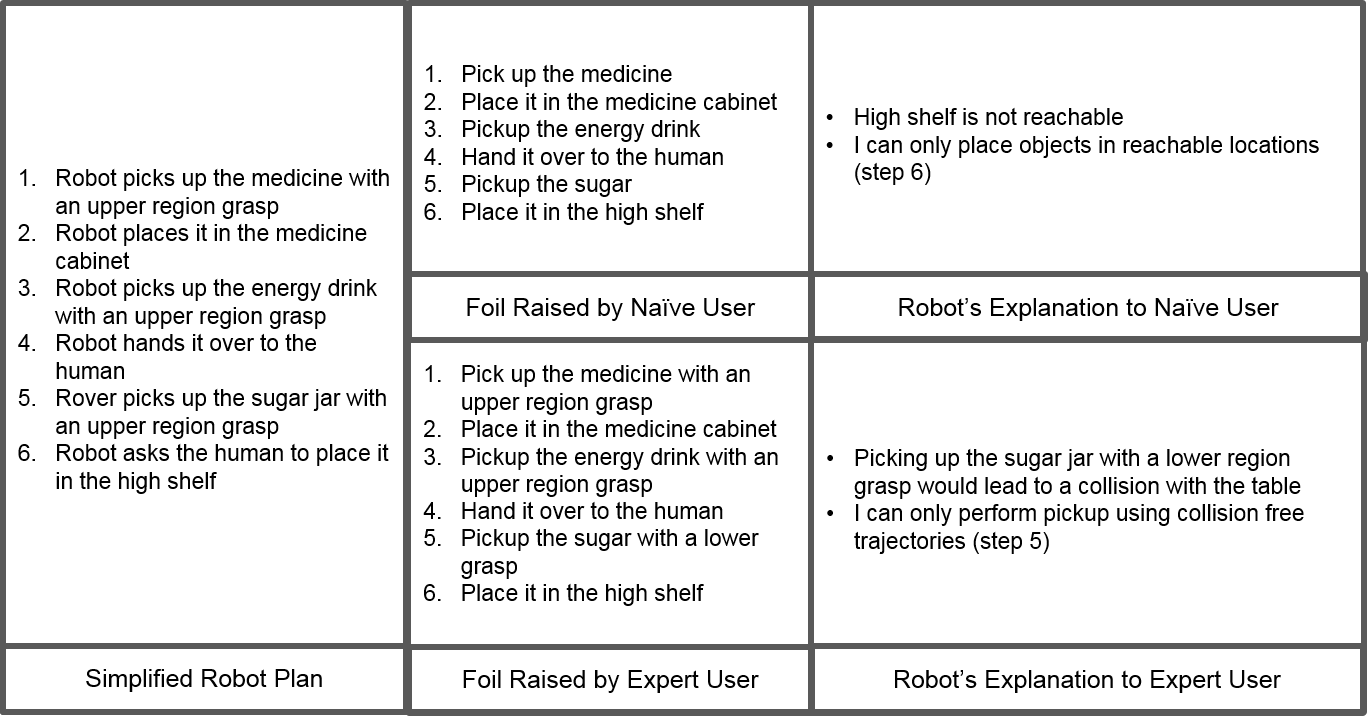}
\caption{{\small The plan and foils used in the scenario.}}
\label{expl_robot_example}
\end{figure}
%\note{
For this demo, we will consider a non proposition conserving model lattice that spans multiple levels of abstractions. Starting out we can convert each of the continuous arguments into geometric symbols (this is similar to the approach used in \cite{srivastava2014combined}). Next, we will further abstract the poses to align with possible regions on the object (i.e., pick up the object from the bottom, middle or top). We will also consider abstractions where we combine the predicates {\small \textsf{is\_putdown\_pose}} and {\small \textsf{collision\_free\_trajectory}} into a single predicate called {\small \textsf{reachable}} predicate and also create new models by dropping arguments from the actions (We only drop an argument when none of the predicates use this argument). Figure \ref{abs-robot} presents an intermediate model where most of the geometric predicates are already abstracted out. There could also be additional non-geometric predicates (like the {\small \textsf{is\_condiment\_type}} predicate) that can be abstracted out.

Figure \ref{expl_robot_example} presents a plan and possible foils that could be generated in this domain. The plan involves the robot placing the energy drink and medicine on its own but relying on the human to complete the place action for sugar jar. The naive user asks merely why the robot doesn't finish the plan on its own while the expert user provides specific grasp that she/he believes can help the robot complete the plan. While the explanation to the naive user relies on the high-level predicate {\small \textsf{reachable}}, the expert explanation is more detailed and relates to the fact that a lower region grasp will result in a collision with the table. We can present such collisions to the human by simulating the trajectories using tools like Rviz \cite{rviz}.

The readers can view the demonstration of the scenario implemented on a fetch robot at \protect{\url{https://youtu.be/qUHg8RABjsw}}. OpenRAVE~\cite{diankov_thesis} was used to compute the trajectories of the robot arm for the pickup and place actions. COLLADA \cite{collada} models of the furniture and the items were created and populated in the OpenRAVE environment using AR markers, the transformation of which was obtained using the ar\_track\_alvar package available in ROS. Resulting OpenRAVE trajectories were then converted into ROS JointTrajectory messages and executed on the robot.%}
\begin{figure}[tbp]
\centering
\begin{tabular}{|l|}
\hline
\\ 
    (:action pickup\\
        :parameters (?x - item ?y - storage)\\
        :precondition (and\\
\hspace{1cm}(in ?x ?y)\\
\hspace{1cm}(handempty)\\
\hspace{1cm}(reachable ?y)\\
        )\\
        :effect (and \\
\hspace{1cm}(not (handempty))\\
\hspace{1cm}(not (in ?x ?y))\\
\hspace{1cm}(holding ?x)\\
\hspace{1cm}(increase (total-cost) 1)\\
        )\\
    )\\
\\
    (:action place\_in\_high\_shelf\\
        :parameters (?x - item ?y - storage)\\
        :precondition (and\\
\hspace{1cm}(is\_condiment\_type ?x)\\
\hspace{1cm}(holding ?x)\\
\hspace{1cm}(reachable ?y)\\
\hspace{1cm}(is\_high\_shelf ?y)\\
        )\\
        :effect (and\\ 
\hspace{1cm}(handempty)\\
\hspace{1cm}(in ?x ?y)\\
\hspace{1cm}(not (holding ?x))\\
\hspace{1cm}(item\_putaway ?x)\\
\hspace{1cm}(increase (total-cost) 1)\\
        )\\
    )\\
\hline
\end{tabular}
\caption{{\small The action definitions for \textsf{pickup} and \textsf{place\_in\_high\_shelf} from an abstract model.}}
\label{abs-robot}
\end{figure}
\section{Related Works}
There is increasing interest within the automated planning community to solve the problem of generating explanations for plans (\cite{danmaga}). Earlier works like \cite{seegebarth,bercher,kambhampati1990classification}  looked at explanations as a way of describing the effects of plans, while works like \cite{sohrabi,meadows} looked at plans itself as explanations for a set of observations. Another approach that has received a lot of interest recently is to view explanations as a way of achieving model reconciliation\cite{explain}. Such explanations are referred to as MRP explanations and this approach postulates that the goal of an explanation is to update the model of the observer so they can correctly evaluate the plans in question.

Similar to MRP, we can also see our explanations as model updates, but we focus on a specific type of update, namely model concretization. Unlike MRP we do not make any assumptions about the availability of human model or the human's computational capabilities. The assumption that we have access to foils help us scale to much larger problems as compared to the original MRP approach. Following the conventions of the original MRP paper, we can see that the explanations studied here are both complete and monotonic. 

The idea of using foils or counterexamples to drive model refinement has also been studied in model checking community under the banner of ``Counter-example Guided Abstraction Refinement" or CEGAR \cite{clarke2000counterexample}. Many planning works have also successfully used CEGAR based methods to generate abstraction heuristics (\cite{seipp2013counterexample,seipp2014diverse}). Even though related, we do not believe that vanilla CEGAR methods can address the problems studied here. Firstly, CEGAR works do not consider model uncertainty which is central to our explanation generation problem. To the best of our knowledge, CEGAR based methods do not assign any costs to refinements, and therefore would not be able to identify the minimal explanation for a given foil set.  Finally, since we are considering sets of foils, it may be prohibitively expensive to follow CEGAR approaches and test each foil in the most concrete model to identify specific faults.

Many abstraction schemes have been proposed for planning tasks (starting with \cite{sacerdoti1974planning}), but in this paper, we mainly focused on state abstractions and %in particular projecting out state fluents. We 
based our formulation on previous works like \cite{srivastava2016metaphysics} and   \cite{backstrom2013bridging}.% that have studied state abstraction. 
 It would be interesting to see how we can extend the approaches discussed in this paper to handle temporal and procedural abstractions (e.g., HLAs \cite{marthi2007angelic}).
\section{Conclusion and Discussion}
In this paper, we investigated the problem of generating explanations when the explainee understands the task model at a lower levels of abstraction. We looked at how we can use explanations as concretization for such scenarios and proposed algorithms for generating minimal explanations. One unique aspect of our approach is the use of foils as a way of capturing human confusion about the problem. This not only helps us formulate more efficient explanation generation methods but also aligns with how humans ask for explanations. 
%The assumption of explicit foils not only affords us the possibility of considering cheaper explanation generation methods but also frees us from the need to make any assumption about the human\textquotesingle s planning capabilities. We merely require that the observer be capable of evaluating the correctness of a given plan. 
Moreover, in most real-world scenarios when we expect someone to explain something, we include the foil in the request for the explanation unless the foil is quite apparent from the context. %Even in cases where the foil is obvious, skipping the foil can lead to confusing explanations.% \cite{tsang2011contrastive}.

Future directions include extending the methods to handle models that are incorrect in addition to being imprecise and looking at other possible methods for abstraction. 
%The discussion in the paper has been centered around state abstraction techniques but planning literature has looked at many other possible ways of abstracting planning tasks
%\footnote{The famous anecdote that the bank robber Willie Sutton replied to the question ``why do you rob banks?" by saying ``that's where the money is" is widely cited as an example of providing explanations when the explainer and explainee do not agree on the foils for the explanations 
%\cite{tsang2011contrastive}}.

% An important future direction for this work would be to consider model noise in addition to model impreciseness. In most realistic scenarios, the human may not only have an imprecise model of the task but also an incorrect one. In such cases, it would be unrealistic to assume that we could have access to a lattice which is guaranteed to contain the human model.

% Most of the discussion in the paper has been centered around state abstraction techniques but planning literature has looked at many other possible ways of abstracting planning tasks many of which are inspired by techniques that are used by humans in our day to day life. For the robot or planner to be able to interact with humans it is important that is able to handle a variety of such abstraction methods.
\section*{Acknowledgments} The authors would like to thank Midhun Pookkottil Madhusoodanan for helping setup the robot demo that is described in the paper. This research is supported in part by the ONR grants N00014161-2892, N00014-13-1-0176, N00014- 13-1-0519, N00014-15-1-2027, and the NASA grant NNX17AD06G.
\bibliographystyle{aaai}
\bibliography{bib}
\end{document}